\documentclass{article}

\usepackage{amsmath}
\usepackage{amssymb}
\usepackage{amsthm}

\newtheorem{thm}{Theorem}
\newtheorem{df}[thm]{Definition}

\usepackage{graphicx}

\title{Covariance and PCA for Categorical Variables}

\author{Hirotaka Niitsuma and Takashi Okada}

\begin{document}

\maketitle

\begin{abstract}
Covariances from categorical variables are defined using a regular simplex expression for categories.
The method follows the variance definition by Gini, and it gives the covariance as a solution of simultaneous equations.
 The calculated results give reasonable values for test data.
A method of principal component analysis (RS-PCA) is also proposed using regular simplex expressions, which allows easy interpretation of the principal components.
The proposed methods apply to variable selection problem of categorical data USCensus1990 data.
The proposed methods give appropriate criterion for the variable selection problem of categorical

\end{abstract}

\section{Introduction }

There are large collections of categorical data in many applications, such as information retrieval, web browsing, telecommunications, and market basket analysis.
While the dimensionality of such data sets can be large, the variables (or attributes) are seldom completely independent.
Rather, it is natural to assume that the attributes are organized into topics, which may overlap, i.e., collections of variables whose occurrences are somehow correlated to each other.

One method to find such relationships is to select appropriate variables 
and to view the data using a method like Principle Components Analysis (PCA) \cite{PCANNbook1996}.
This approach gives us a clear picture of the data using KL-plot of the PCA.
However, the method is not settled for the data including categorical data.
Multinomial PCA \cite{MPCA2002Buntine} is analogies to PCA for handling discrete or categorical data.
However, Multinomial PCA is a method based on the parametric model and it is difficult to construct a KL-plot for the estimated result.
Multiple Correspondence Analysis (MCA) \cite{MCA1998Clausen}
is analogous to PCA and can handle discrete categorical data.
MCA is also known as homogeneity analysis, dual scaling, or reciprocal averaging. 
The basic premise of the technique is that complicated multivariate data can be made more accessible by displaying their main regularities and patterns as plots ("KL-plot") .
MCA is not based on a parametric model and can give a "KL-plot" for the estimated result.
In order to represent the structure of the data, sometimes we need to ignore meaningless variables.
However, MCA does not give covariances or correlation coefficients between a pair of categorical variables.
It is difficult to obtain criteria for selecting appropriate categorical variables using MCA.

Symbolic Data Analysis\cite{IchinoSDA88,IchinoSDA94} is one of the methods to give multivariate descriptive for categorical data.
However we forces on more intuitive method which can give an understandable plot like K-L plot.

In this paper, we introduce the covariance between a pair of categorical variables using the regular simplex expression of categorical data. This can give a criterion for selecting appropriate categorical variables.
We also propose a new PCA method for categorical data.

\begin{table}
\caption{Fisher's data
\label{table:Fisher'sdata}
}
\centering
\fbox{
\begin{tabular}{cccccc}
$x_{eye}$ $\backslash$ $x_{hair}$  &fair &red &medium &dark &black\\
blue &326& 38& 241& 110& 3\\
light &688& 116& 584& 188& 4\\
medium &343 &84 &909 &412 &26\\
dark &98 &48 &403 &681 &85
\end{tabular}
}
\end{table}

\section{Gini's Definition of Variance and its Extension}
\label{GiniDefinitionExtension}

Let us consider the contingency table shown in Table \ref{table:Fisher'sdata}, which is known as Fisher's data  \cite{fishersData1940} on the colors of the eyes and hair of the inhabitants of Caithness, Scotland.
The table represents the joint population distribution of the categorical variable for eye color $x_{eye}$ and the categorical variable for hair color $x_{hair}$:
\begin{eqnarray}
x_{hair} &\in& \{\mbox{ fair red medium dark black}  \}
\nonumber
\\
x_{eye} &\in& \{\mbox{ blue light medium  dark}  \}.
\end{eqnarray}
Before defining the covariances among such categorical variables, 
$\sigma_{hair,eye}$, let us
consider the variance of a categorical variable. 
Gini successfully defined the variance for categorical data \cite{Gini71}.
\begin{eqnarray}
\sigma_{ii}=\frac{1}{2N^2}\sum_{a=1}^N \sum_{b=1}^N (x_{ia}-x_{ib})^2
\label{eq:def_gini_index}
\end{eqnarray}
where, $\sigma_{ii}$
is the variance of the $i$-th variable, $x_{ia}$ is the value of $x_i$ for the $a$-th instance, and $N$ is the number of instances.
The distance of a categorical variable between instances is defined as 
$x_{ia}-x_{ib} = 0$ if their values are identical,  and $= 1$ otherwise.
A simple extension of this definition to the covariance $\sigma_{ij}$ by replacing $(x_{ia} - x_{ib})^2$ to $(x_{ia} -x_{ib})(x_{ja} -x_{jb})$
does not give reasonable values for the covariance
$\sigma_{ij}$
 \cite{Okada2000cov}.
In order to avoid this difficulty, we extended the definition 
based on scalar values, $x_{ia}-x_{ib}$, to a new definition using a vector expression \cite{Okada2000cov}.
The vector expression for a categorical variable with three categories $x_i \in \{ r_1^i , r_2^i , r_{3}^i \}$ was defined by placing these three categories at the vertices of a regular triangle.

A regular simplex can be used for a variable with more than four categories. 
This is a straightforward extension of a regular triangle when the dimension of space is greater than two.
For example, a regular simplex in the 3-dimensional space is a regular tetrahedron.
Using a regular simplex, we can extend and generalize the definition of covariance to

\begin{df}
\label{defintition:sigmaVar_ij}
The covariance between a categorical variable
$x_i \in \{ r_1^i , r_2^i , ... r_{k_i}^i \}$
 with $k_i$ categories and
a categorical variable 
$x_j \in \{ r_1^j , r_2^j , ... r_{k_j}^j \}$
with $k_j$ categories
is defined as 
\begin{eqnarray}
&&
\sigma_{ij}=
\max_{L^{ij}}
(
\frac{1}{2N^2}
\nonumber
\\
&&
\sum_{a=1...N} \sum_{b=1...N} 
({\bf v}^{k_i}(x_{ia})-  {\bf v}^{k_i}(x_{ib}) )
L^{ij}
({\bf v}^{k_j}(x_{ja})-  {\bf v}^{k_j}(x_{jb}) )^t
),
\label{eq:def_sigma_ij_max}
\end{eqnarray}
where ${\bf v}^n(r_k)$ is the position of the $k$-th vertex of a regular $(n-1)$-simplex 
{\rm  \cite{hypertetrahedronBuekenhout98}}.
$r_k^i$ denotes the $k$-th element of the $i$-th categorical variable $x_i$.
$L^{ij}$ is a unitary matrix expressing the rotation between the regular
 simplexes for $x_i$ and $x_j$.
 \end{df}

Definition \ref{defintition:sigmaVar_ij} includes a procedure to maximize the covariance.
Using Lagrange multipliers, this procedure can be converted into a simpler problem of simultaneous equations, which can be solved using the Newton method. The following theorem enables this problem transformation.

\begin{thm}
\label{sigma_ij_def}
The covariance between categorical variable $x_i$ with $k_i$ categories and categorical variable $x_j$ with $k_j$ categories is expressed by
\begin{eqnarray}
&&  \sigma_{ij}=trace (A^{ij} {L^{ij}}^t) 
\label{eq:def_sigma_ij_byA},
\end{eqnarray}
where 
$A^{ij}$ is $(k_i-1) \times (k_j-1) $ matrix :
\begin{eqnarray}
A^{ij}  =
\frac{1}{2 N^2}
\sum_a \sum_b 
({\bf v}^{k_i}(x_{ia})  -{\bf v}^{k_i}(x_{ib})  )^t
({\bf v}^{k_j}(x_{ja})   -{\bf v}^{k_j}(x_{jb}) )
\label{eq:def_A_ij_by_n_simplex}.
\end{eqnarray}

$L^{ij}$ is given by the solution of the following simultaneous equations.

\begin{eqnarray}
&&A^{ij} {L^{ij}}^t =( A^{ij} {L^{ij}}^t )^t
\nonumber\\
 &&L^{ij} {L^{ij}}^t={\bf E}
\label{eq:EquationLij}
\end{eqnarray}
\end{thm}

\begin{proof}
Here, we consider the case where $k_i=k_j$ for the sake of simplicity.
Definition \ref{defintition:sigmaVar_ij} gives a conditional
maximization problem :
\begin{eqnarray}
\sigma_{ij}=
\max_{L^{ij}}  && 
\frac{1}{2 N^2}
\sum_a \sum_b 
({\bf v}^{k_i}(x_{ia})  -{\bf v}^{k_i}(x_{ib})  )
L^{ij}
({\bf v}^{k_j}(x_{ja})   -{\bf v}^{k_j}(x_{jb}) )^t
\nonumber\\
 \mbox{subject to }&&L^{ij} {L^{ij}}^t=  {\bf E}
\label{eq:cov_def_Start}
\end{eqnarray} 
The introduction of Lagrange multipliers $\Lambda$ for the constraint $L^{ij} {L^{ij}}^t={\bf E}$ gives the Lagrangian function:
\[
V=trace( A^{ij}{L^{ij}}^t)
- trace( \Lambda^t 
L^{ij}  {L^{ij}}^t-{\bf E}
),
\]
where $\Lambda$ is $k_i \times k_i$ matrix.
A stationary point of the Lagrangian function $V$ is a solution of the
 simultaneous equations (\ref{eq:EquationLij}).
 \end{proof}

Instead of maximizing (\ref{eq:def_sigma_ij_max}) with constraint $L^{ij}  {L^{ij}}^t={\bf E}$  , we can get the covariance by solving the equations (\ref{eq:EquationLij}), which can be solved easily using the Newton method
.
More efficient way to compute the covariance is
the following Singular Value Decomposition of matrix $A^{ij}$.

\begin{thm}
Let Singular Value Decomposition of matrix $A^{ij}$ be 

\[
A^{ij} = U D V^t.
\]
The solution of the maximization problem (\ref{eq:def_sigma_ij_max}) is given
\[
L^{ij} = U V^t,
\]
\[
\sigma_{ij}=trace(D).
\]
 \end{thm}


Application of this method to Table \ref{table:Fisher'sdata} gives 
\begin{eqnarray}
\sigma_{hair,hair}= 0.36409, \sigma_{eye,hair}=0.081253 , \sigma_{eye,eye}=0.34985
\end{eqnarray}
We can derive a correlation coefficient using the covariance and variance values of categorical variables in the usual way. 
The correlation coefficients for $x_{eye},x_{hair}$ for Table \ref{table:Fisher'sdata} 
is 0.2277.

\section{ Principal Component Analysis }

\subsection{Principal Component Analysis of Categorical Data using Regular Simplex\,(RS-PCA)}

Let us consider categorical variables $x_1,x_2...x_J$.
For the $a$-th instance, $x_i$ takes value $x_{ia}$.
Here, we represent $x_{ia}$ by the vector of vertex coordinates ${\bf v}^{k_i}( x_{ia} )$.
Then, the values of all the categorical variables $x_1,x_2...x_J$ for the $a$-th instance can be represented by the concatenation of the vertex coordinate vectors of all the categorical variables:
\begin{eqnarray}
{\bf x}(a)
=( {\bf v}^{k_1}( x_{1a} ) ,{\bf v}^{k_2}( x_{2a}  ),...,{\bf v}^{k_J}( x_{Ja} )       ).
\label{eq:defTotalVecx}
\end{eqnarray}
Let us call this concatenated vector the {\it List of Regular Simplex Vertices }(LRSV).
The covariance matrix of LRSV can be written as
\begin{eqnarray}
{\cal A}&=&
\frac{1}{N}
\sum_{a=1}^N
({\bf x}({a})  -\bar{\bf x}  )^t
({\bf x}({a})  -\bar{\bf x}   )
=
\left[ 
\begin{array}{c|c|c|c}
 A^{11}  &A^{12}  & ...  & A^{1J}  \\
\hline
A^{21}  & A^{22}  & ...  & A^{2J}  \\
\hline
 ... & ... & ...  &... \\
\hline
 A^{J1}  &  A^{J2} & ...  & A^{JJ}  \\
\end{array}
\right].
\label{eq:delCalA}
\end{eqnarray}
where 
$\bar{\bf x} = \frac{1}{N}\sum_{a=1}^N {\bf x}({a}) $
 is an average of the LRSV.
The equation (\ref{eq:delCalA}) shows the covariance matrix of LRSV. 
Since its eigenvalue decomposition can be regarded as a kind of Principal Component Analysis (PCA) on LRSV, 
we call it the {\it Principal Component Analysis using the Regular Simplex for categorical data} (RS-PCA).

When we need to interpret an eigenvector from RS-PCA, it is useful to express the eigenvector as a linear combination of the following vectors.
The first basis set, $d$, shows vectors from one vertex to another vertex in the regular simplex. 
The other basis set, $c$, show vectors from the center of the regular simplex to one of the vertices.
\begin{eqnarray}
&&{\bf{d}}^{k_j}(a \rightarrow b) =  {\bf{v}}^{k_j}(b) - {\bf{v}}^{k_j}(a) \hspace{1cm} a,b=1,2...k_j 
\\
&&{\bf{c}}^{k_j}(a)=  {\bf{v}}^{k_j}(a) - \frac{\sum_{b=1}^{k_j }{\bf{v}}^{k_j}(b) }{k_j}
\hspace{1cm}
a=1,2...k_j
\end{eqnarray}
Eigenvectors defined in this way change their basis set depending on its direction to the regular simplex, but this has the advantage of allowing us to grasp its meaning easily. 
For example, the first two principal component vectors from the data in Table \ref{table:Fisher'sdata} are expressed using the following linear combination. 
\begin{eqnarray}
{\bf v}^{ rs-pca}_1
&&=
-0.63 \cdot {\bf{d}}^{eye}(medium \rightarrow light)
-0.09 \cdot {\bf{c}}^{eye}(blue)
-0.03 \cdot {\bf{c}}^{eye}(dark)
\nonumber\\
&&
-0.76 \cdot {\bf{d}}^{hair}(medium \rightarrow fair)
+0.07 \cdot {\bf{d}}^{hair}(dark \rightarrow medium)
\label{eq:firstPrincipalComponentExplain}
\\
{\bf v}^{ rs-pca}_2
&&=
0.64 \cdot {\bf{d}}^{eye}(dark \rightarrow light)
-0.13 \cdot {\bf{d}}^{eye}(medium \rightarrow light)
\nonumber\\
&&
-0.68 \cdot {\bf{d}}^{hair}(dark \rightarrow medium)
+0.30 \cdot {\bf{c}}^{hair}(fair)
\label{eq:secondPrincipalComponentExplain}
\end{eqnarray}

This expression shows that the axis is mostly characterized by the difference between 
$x^{eye}=light$ and $x^{eye}=medium$ values, and 
the difference between $x^{hair}=medium$ and $x^{hair}=fair$ values.
The KL-plot using these components is shown in Figure \ref{fig:FisherDataRS-PCA} for Fisher's data.
In this figure, the lower side is mainly occupied by data with values: $x^{eye}=medium$ or $x^{hair}=medium$.
The upper side is mainly occupied by data with values $x^{eye}=light$ or $x^{hair}=fair$.
Therefore, we can confirm that 
$({\bf{d}}^{eye}(medium \rightarrow light) + {\bf{d}}^{hair}(medium \rightarrow fair))$
is the first principal component.
In this way, we can easily interpret the data distribution on the KL-plot when we use the RS-PCA method.

Multiple Correspondence Analysis (MCA) \cite{GowerBookBiplot96} provides a similar PCA methodology to that of RS-PCA.
It uses the representation of categorical values as an 
indicator matrix (also known as a dummy matrix).
MCA gives a similar KL-plot.
However, the explanation of its principal components is difficult, 
because their basis vectors contain one redundant dimension compared to the regular simplex expression.
Therefore, 
a conclusion from MCA can only be drawn after making a great effort to inspect the KL-plot of the data.

\begin{figure}
 {
 
\includegraphics
[width=1.0\textwidth] 
 {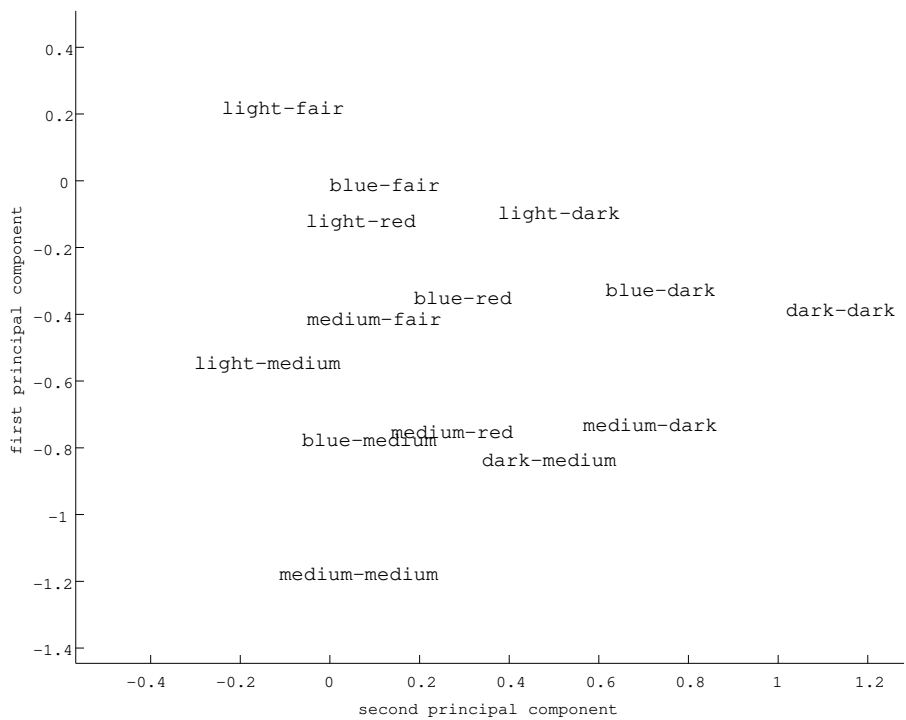}
}
\caption{
KL-plot of Fisher's data calculated using RS-PCA.
A point is expressed by a pair of eye and hair categories: $x^{eye}-x^{hair}$.
}

\label{fig:FisherDataRS-PCA}
\end{figure}

\section{Experimental Results}

We evaluated the performance of our algorithms on a 
1990 US census dataset\\
(http://kdd.ics.uci.edu/databases/census1990/USCensus1990.html).
1990 US census dataset is a multivariate categorical data which describes census data of US.
The data set includes 68 discretized attributes such as age, income, occupation, work status, etc.
In this experiment, we ignore categorical variable "iOthrserv", since this variable has same value on almost all entries.
We randomly selected 3k entries from the 2.5M available entries in the entire data set, and apply our method to 67 discretized attributes.
Table \ref{table:categCovVarUSCensus1990} and \ref{table:categCorrelationCoefficientsUSCensus1990} show covariances and correlation coefficients respectively, among some categorical variables of 1990 US census dataset given by equation (\ref{eq:def_sigma_ij_max}).
Figure \ref{fig:usCensus1990index_eigenvalue} shows eigenvalues of a
covariance matrix for the 67 categorical variables vs mode number.
In this figure, only top 20 eigenvalues have large values.
This means, almost 20 categorical variables are sufficient to explain 1990 US census dataset.

Figure \ref{fig:usCensus1990VarCovEig} is K-L plot of categorical
variables.
In this figure, we can see categorical variable iRlabor corresponds to
1st principal component and iSex corresponds to 2nd principal component,
since these variables close to correspond axes.

In the following, results of RS-PCA are compared focused on these two categorical variables iRlabor and
iSex.
Figure \ref{fig:usCensus1990rspcaFull} plots result of RS-PCA using all
67 variables.
Figure \ref{fig:usCensus1990RSPCAExi1} is RS-PCA result using first top 20
variables:iRlabor, iSex , and so on.
Almost same structure to the result using all variables is appeared in this
figure.
Figure \ref{fig:usCensus1990rspcaMod} is RS-PCA result using rest 37-67th
principal components.
In this figure, we cannot find similar structure.
Figure \ref{fig:usCensus1990rspcaPc1To5} is RS-PCA result using top 5
principal components.
We can find similar structure to the result using all variables.
This results intend that abstracted data structure can described by only
5 variables.

The above mentioned results show that our method can use for variables
selection of categorical data.

\begin{figure}
 {\includegraphics[width=0.9\textwidth]
{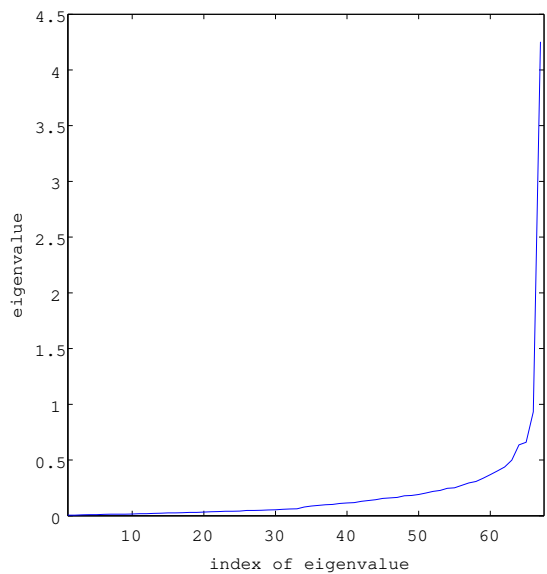}
 }
\caption{eigenvalue vs mode number}
\label{fig:usCensus1990index_eigenvalue}
\end{figure}

\begin{figure}
 {
 \includegraphics[width=0.9\textwidth]
 {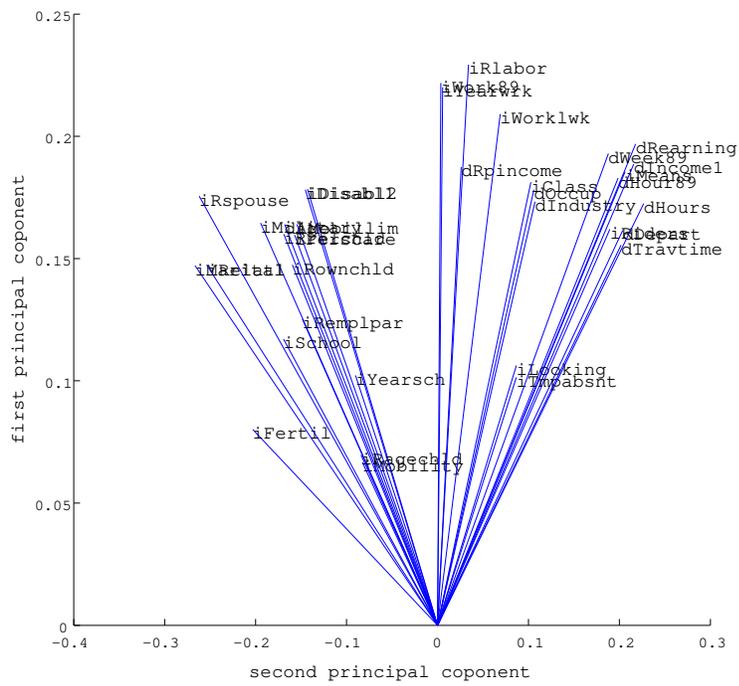}
}
\caption{KL-plot of USCensus 1990 using all variables}
\label{fig:usCensus1990VarCovEig}
\end{figure}

\begin{table}
 \caption{covariance of USCensus1990}
 \label{table:categCovVarUSCensus1990}
\fbox{
\begin{tabular}{cccccccccc}
 &dAnc1 	    &dAnc2 	&iClass 	&dHisp      &dIndu     &iLan1 &dOcc 	&dPOB	&iSex \\
dAncstry1&	0.332&		0.076&		0.016&		0.017&		0.018&		0.05&	0.021&	0.039&	0.003 \\
dAncstry2& 	0.076&		0.226&		0.012&		0.006&		0.012&		0.024&	0.014&	0.015&	0.002 \\
iClass	&	0.016&		0.012&		0.332&		0.003&		0.206&		0.038&	0.19&	0.007&	0.019 \\
dHispanic&	0.017&		0.006&		0.003&		0.033&		0.003&		0.016&	0.003&	0.013&	0\\
dIndustry& 	0.018&		0.012&		0.206&		0.003&		0.392&		0.036&	0.209&	0.008&	0.029\\
iLang1&		0.05&		0.024&		0.038&		0.016&		0.036&		0.166&	0.038&	0.042&	0\\
dOccup&		0.021&		0.014&		0.19&		0.003&		0.209&		0.038&	0.38&	0.008&	0.036\\
dPOB&		0.039&		0.015&		0.007&		0.013&		0.008&		0.042&	0.008&	0.08&	0.001\\
iSex&		0.003&		0.002&		0.019&		0&			0.029&		0&		0.036&	0.001&	0.249
\end{tabular}
}
\end{table}

\begin{table}
\caption{Correlation coefficients of USCensus1990}
\label{table:categCorrelationCoefficientsUSCensus1990}
\fbox{
\begin{tabular}{cccccccccc}
 			&dAnc1 	&dAnc2 	&iClass 	&dHisp &dIndu &iLan1 &dOcc 	&dPOB	&iSex \\
dAncstry1&	1&			0.28&		0.049&		0.164&		0.05&		0.215&	0.06&	0.24&	0.012\\
dAncstry2&	0.28&		1&			0.044&		0.074&		0.042&		0.126&	0.048&	0.117&	0.012\\
iClass	&	0.049&		0.044&		1&			0.037&		0.571&		0.163&	0.536&	0.047&	0.067\\
dHispanic&	0.164&		0.074&		0.037&		1&			0.033&		0.223&	0.035&	0.261&	0.01\\
dIndustry&	0.05&		0.042&		0.571&		0.033&		1&			0.143&	0.542&	0.046&	0.093\\
iLang1&		0.215&		0.126&		0.163&		0.223&		0.143&		1&		0.154&	0.369&	0.001\\
dOccup&		0.06&		0.048&		0.536&		0.035&		0.542&		0.154&	1&		0.048&	0.117\\
dPOB&		0.24&		0.117&		0.047&		0.261&		0.046&		0.369&	0.048&	1&		0.01\\
iSex&		0.012&		0.012&		0.067&		0.01&		0.093&		0.001&	0.117&	0.01&	1
\end{tabular}
 }
\end{table}

\begin{figure}
 {\includegraphics[width=0.9\textwidth]
{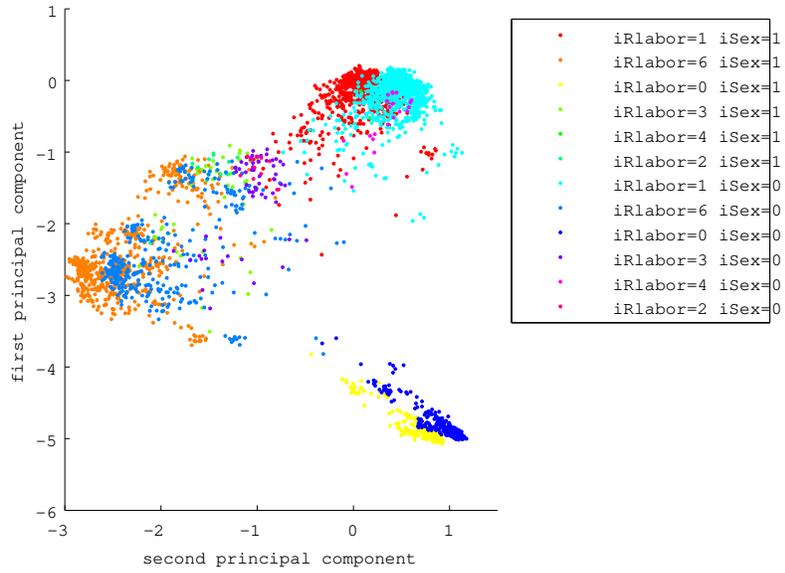}}
\caption{RS-PCA   of USCensus1990 }
\label{fig:usCensus1990rspcaFull}
\end{figure}

\begin{figure}
 {
 \includegraphics[width=0.9\textwidth]
 {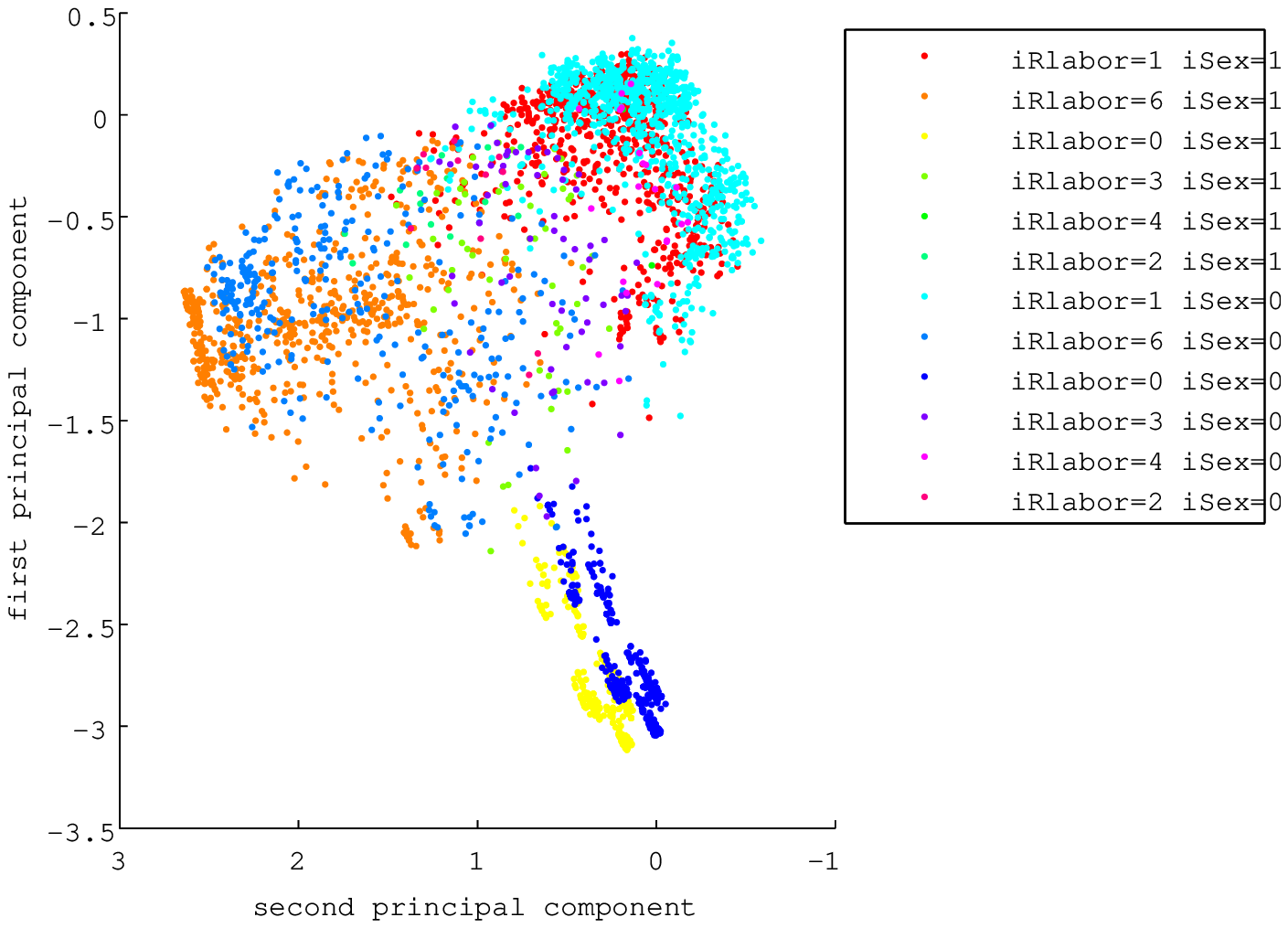}
}
\caption{RS-PCA using top 20 principal components}
\label{fig:usCensus1990RSPCAExi1}
\end{figure}

\begin{figure}
 {
 \includegraphics[width=0.9\textwidth]
{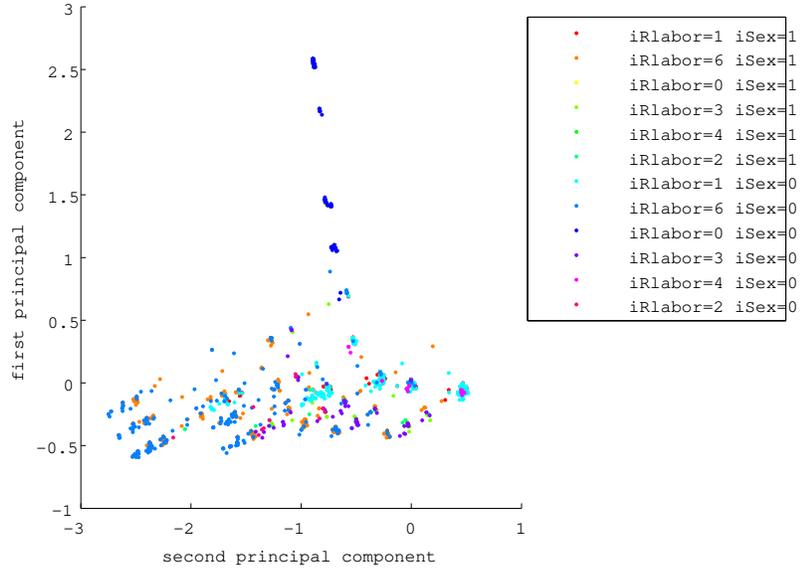}
}
\caption{RRS-PCA using 37-67th principal components }
\label{fig:usCensus1990rspcaMod}
\end{figure}

\begin{figure}
 {
 \includegraphics[width=0.9\textwidth]
{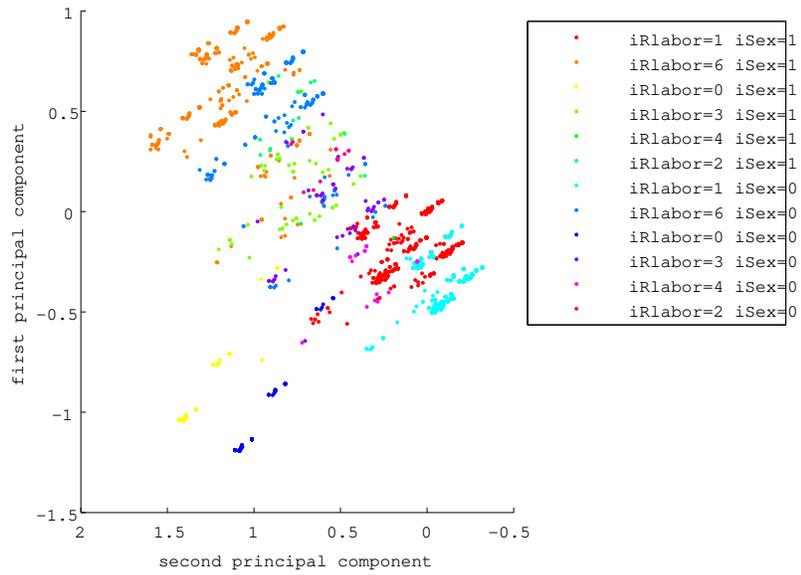}}

\caption{RRS-PCA  using top 5 principal components  }
\label{fig:usCensus1990rspcaPc1To5}
\end{figure}

\section{Conclusion}

We studied the covariances between a pair of categorical variables based
on Gini's definition of the variance for categorical data. The
introduction of the regular simplex expression for categorical values
enabled a reasonable definition of covariances, and an algorithm for
computing the covariance was proposed. The regular simplex expression
was also shown to be useful in the PCA analysis. We showed these merits
through numerical experiments using Fisher's data and USCensus1990 data.
In these experiments, our method applied to variable selection problem of categorical
data. The experiments showed our method gives appropriate criterion for variable selection.
 

\section*{Acknowledgment}
This research was partially supported by the Ministry of Education, Culture, Sport, Science and Technology, of Japan, with a Grant-in-Aid for Scientific Research on Priority Areas, 13131210 and a Grant-in-Aid 
for Scientific Research (A) 14208032.

\bibliographystyle{abbrv}
\bibliography{cov}

\end{document}